\documentclass[10pt]{article}
\usepackage{amsmath,amsfonts,amssymb,latexsym,graphicx}
\usepackage{amsthm}

\usepackage{etoolbox,mathtools,stmaryrd,cite,enumitem} % enumerate
\usepackage{pict2e} % to see short lines in the picture environment
\usepackage[shortcuts]{extdash} % for hyphenating "p-variable" etc.

\newtoggle{CONF}
\newtoggle{arXiv}
\newtoggle{TR}
\newtoggle{false} % always false
\newtoggle{true}\toggletrue{true} % always true

\toggletrue{arXiv}
\toggletrue{TR}

\allowdisplaybreaks[4]

\usepackage{environ}
\NewEnviron{wideformula}{%
  \iftoggle{CONF}{\begin{equation}\BODY\end{equation}}%
     {\begin{multline}\BODY\end{multline}}}
\NewEnviron{wideformula*}{%
  \iftoggle{CONF}{\begin{equation*}\BODY\end{equation*}}%
    {\begin{multline*}\BODY\end{multline*}}}
\newcommand{\NarrowLineBreak}{\iftoggle{CONF}{}{\\}}
\NewEnviron{widish}{\iftoggle{CONF}{$\BODY$ }{\[\BODY\]}}
\NewEnviron{widish.}{\iftoggle{CONF}{$\BODY$. }{\[\BODY.\]}}

\usepackage[pdfpagemode=UseNone,pdfstartview=FitH]{hyperref}

\usepackage{color}
\newcommand{\red}[1]{\textcolor{red}{#1}}

\newtheorem{theorem}{Theorem}

\newtheorem{lemma}[theorem]{Lemma}
\newtheorem{corollary}[theorem]{Corollary}

\theoremstyle{definition}%[section]

\theoremstyle{remark}
\newtheorem{remark}[theorem]{Remark}

\newcommand{\Mod}[1]{\ (\mathrm{mod}\,#1)}

\newcommand{\dd}{\,\mathrm{d}}  % differential
\renewcommand{\d}{\mathrm{d}}   % differential for use in parentheses
\newcommand{\e}{\mathrm{e}}    % mathematical constant \approx 2.72

\newcommand{\ER}{\mathcal{E}^{\mathrm{R}}}  % randomness e-variables
\newcommand{\EX}{\mathcal{E}^{\mathrm{X}}}  % exchangeability e-variables
\newcommand{\PR}{\mathcal{P}^{\mathrm{R}}}  % randomness p-variables
\newcommand{\PX}{\mathcal{P}^{\mathrm{X}}}  % exchangeability p-variables
\newcommand{\EtR}{\mathcal{E}^{\mathrm{tR}}} % train-invariant randomness e-variables
\newcommand{\EtX}{\mathcal{E}^{\mathrm{tX}}} % train-invariant exchangeability e-variables
\newcommand{\PtR}{\mathcal{P}^{\mathrm{tR}}} % train-invariant randomness p-variables
\newcommand{\PtX}{\mathcal{P}^{\mathrm{tX}}} % train-invariant exchangeability p-variables
\newcommand{\EiR}{\mathcal{E}^{\mathrm{iR}}} % invariant randomness e-variables
\renewcommand{\i}{^{\mathrm{i}}}  % making invariant
\renewcommand{\t}{^{\mathrm{t}}}  % making train-invariant
\newcommand{\X}{^{\mathrm{X}}}   % making exchangeable
\newcommand{\tX}{^{\mathrm{tX}}} % making conformal

\newcommand{\EEE}{\mathcal{E}} % all variables

\title{Universality of conformal prediction under the assumption of randomness}
\author{Vladimir Vovk}

\begin{document}
\maketitle
\begin{abstract}%
  Conformal predictors provide set or functional predictions
  that are valid under the assumption of randomness,
  i.e., under the assumption of independent and identically distributed data.
  The question asked in this paper is whether there are predictors
  that are valid in the same sense under the assumption of randomness
  and that are more efficient than conformal predictors.
  The answer is that the class of conformal predictors is universal
  in that only limited gains in predictive efficiency are possible.
  The previous work in this area has relied on the algorithmic theory of randomness
  and so involved unspecified constants,
  whereas this paper's results are much more practical.
  They are also shown to be optimal in some respects.

  The version of this paper at \url{http://alrw.net} (Working Paper 43)
  is updated most often.
\end{abstract}

\section{Introduction}
\label{sec:introduction}

The main assumption of machine learning is that of \emph{randomness},
i.e., the assumption that the observations
are independent and identically distributed (IID).
The method of conformal prediction \cite{Angelopoulos/etal:arXiv2411,Vovk/etal:2022book}
allows us to complement predictions output by standard machine-learning algorithms
by some measures of confidence in those predictions
that are valid under the assumption of randomness.
However, conformal prediction only uses the assumption of exchangeability,
which is weaker than randomness.
A natural question is how limiting this is;
does conformal prediction lose much
by not using the full strength of the assumption of randomness
and only using exchangeability?
The main message of this paper
and the earlier paper by Nouretdinov et al.\ \cite{Nouretdinov/etal:2003ALT}
is that the answer is ``no'' (we don't lose much),
at least in the first approximation.
In the terminology of \cite{Nouretdinov/etal:2003ALT},
conformal prediction is universal.

The main limitation of the pioneering paper \cite{Nouretdinov/etal:2003ALT}
is that it is based on the algorithmic theory of randomness.
Because of that, Nouretdinov et al.'s results involve unspecified constants
and, therefore, are never applicable to practical machine learning.
This paper's results are more precise in several respects,
but the main one is that they only involve fully specified constants
and, therefore, open up the possibility of quantifying limitations of conformal predictors.
Other advances of this paper as compared with \cite{Nouretdinov/etal:2003ALT}
are that our results are not restricted to the case of classification,
but in the case of classification they are stronger
and are complemented by optimality results.

To demonstrate that conformal prediction does not lose much
by not using the full strength of the assumption of randomness,
this paper introduces the most general class of predictors,
``randomness predictors'',
which produce predictions of the same kind as conformal predictors
that are only required to be valid under the assumption of randomness.
There are many more randomness predictors than conformal predictors,
and our question is whether conformal predictors are as good
as arbitrary randomness predictors.
The answer given in this paper is a qualified yes:
every randomness predictor $P$ can modelled by a conformal predictor $P'$
so that the predictions output by $P'$ are almost as good as those output by $P$.
It is unclear whether the difference between conformal and randomness prediction
can be usefully exploited at all\iftoggle{CONF}{}{ \cite{Vovk:arXiv2503}}.

One simplifying assumption made in this paper
is that it concentrates on predictors that are \emph{train-invariant},
i.e., invariant w.r.\ to permutations of the training sequence.
(We will remove this assumption only in Appendix~\ref{app:general}.)
The assumption of train-invariance for predictors
is extremely natural under the assumption of randomness for the data.
It is reflected in the standard expression for a training sequence
being ``training set'' in machine learning;
since ``set'' implies the lack of order,
this expression is only justified for train-invariant predictors
(and even in this case it is not justified completely;
it would have been more accurate to say ``training bag'').
In fact, conformal predictors can be defined as train-invariant predictors
that are valid under exchangeability
\cite[Proposition~1]{Nouretdinov/etal:2003ALT}.
This will be discussed in detail in Sect.~\ref{sec:definitions},
where we will also see that the requirement of train-invariance
is justified by general principles of statistical inference.

Conformal prediction is usually presented as a method of \emph{set prediction}
\cite[Part~I]{Vovk/etal:2022book},
i.e., as a way of producing prediction sets (rather than point predictions).
Another way to look at a conformal predictor is as a way of producing a p-value function
(discussed, in a slightly different context, in, e.g., \cite{Fraser:2019}),
which is a function mapping each possible label $y$ of a test object
to the corresponding conformal p-value.
In analogy with ``prediction sets'', we may call such p-value functions
``prediction functions''.
The prediction set $\Gamma^{\alpha}$ corresponding to a prediction function $f$
and a significance level $\alpha\in(0,1)$
(our target probability of error)
is the set of all labels $y$ such that $f(y)>\alpha$.
A standard property of validity for conformal predictors
is that $\Gamma^{\alpha}$ makes an error (fails to cover the true label)
with probability at most $\alpha$;
it is implied by the conformal p-values being bona fide p-values
(under suitable assumptions, such as data exchangeability).

To establish connections between conformal and randomness predictors
we will use conformal e-predictors \cite{Vovk:2025PR-local},
which are obtained by replacing p-values with e-values
(for the definition of e-values, see, e.g., \cite{Grunwald/etal:2024,Ramdas/Wang:book,Vovk/Wang:2021}
or Sect.~\ref{sec:definitions} below).
Conformal e-predictors output e-value functions $f$ as their prediction functions.
Such functions $f$ can also be represented in terms of the corresponding prediction sets
$\Gamma^{\alpha}:=\{y\mid f(y)<\alpha\}$,
where $\alpha\in(0,\infty)$ is the significance level
(notice that now we exclude the labels with large e-values from the prediction set,
which is opposite to what we did for p-values).
However, the property of validity of conformal e-predictors
is slightly more difficult to state in terms of prediction sets:
now validity means that the integral of the probability of error for $\Gamma^{\alpha}$
over $\alpha\in(0,\infty)$ does not exceed 1
\cite[end of Appendix~B]{Vovk:2025PR-local}.
This implies that the probability of error for $\Gamma^{\alpha}$
is at most $1/\alpha$,
but this simple derivative property of validity is much weaker.

Conformal e-predictors are not only a useful technical tool,
but we can also use them for prediction directly.
In Shafer's opinion \cite{Shafer:2021},
e-values are even more intuitive than p-values.
Because of the importance of e-predictors,
in the rest of this paper we will use the word ``predictor''
in combinations such as ``conformal predictor'' and ``randomness predictor''
generically, including both p-predictors (standard predictors based on p-values)
and e-predictors (predictors based on e-values);
in particular, we will never drop ``p-'' in ``p-predictor''.
This is a potential source of confusion,
and the reader should keep in mind that the usual notion of conformal predictors
corresponds to conformal p-predictors in this paper.

\begin{figure}[bt]
  \begin{center}
    \unitlength 0.50mm
    \begin{picture}(50,50)(-10,-10)  % the size of the box and the coordinates of the bottom left corner
      \thinlines
      \put(34,0){\vector(-1,0){28}}  % the bottom embedding
      \put(40,5){\line(0,1){30}}     % the right connection (calibration)
      \put(34,40){\vector(-1,0){28}} % the top embedding
      \put(0,5){\line(0,1){30}}      % the left connection (calibration)
      % and now the notation for the classes:
      \put(-6,35){\framebox(12,10)[cc]{$\EtR$}}  % top left corner
      \put(-6,-5){\framebox(12,10)[cc]{$\PtR$}}  % bottom left corner
      \put(34,-5){\framebox(12,10)[cc]{$\PtX$}}  % bottom right corner
      \put(34,35){\framebox(12,10)[cc]{$\EtX$}}  % top right corner
    \end{picture}
  \end{center}
  \caption{A square representing the main classes of predictors considered in this paper:
    $\EtR$ (train-invariant randomness e-predictors),
    $\EtX$ (conformal e-predictors),
    $\PtR$ (train-invariant randomness p-predictors),
    $\PtX$ (conformal p-predictors).%
    \label{fig:simple}}
\end{figure}

We start in Sect.~\ref{sec:definitions} from the main definitions,
including those of conformal and randomness predictors.
The main classes of predictors that we are interested in in this paper
are shown in Figure~\ref{fig:simple}.
All four classes are train-invariant.
An arrow going from class $A$ to class $B$ means embedding:
$A\subseteq B$.
Section~\ref{sec:results} is devoted to the main results,
first for e-predictors in Subsect.~\ref{subsec:e}
and then for p-predictors in Subsect.~\ref{subsec:p}.
A line between classes $A$ and $B$ in Figure~\ref{fig:simple}
means the possibility of transformations
from $f\in A$ to $f'\in B$ and vice versa;
such transformations, called ``calibration'',
will be discussed at the beginning of Subsect.~\ref{subsec:p}
and will serve as a way of deducing results for p-predictors
from those for e-predictors.
In the two subsections,
we establish the predictive efficiency of conformal predictors
among randomness predictors in both e- and p-versions.
Namely, the prediction functions for conformal predictors turn out to be competitive on average
with the prediction functions for any randomness predictors,
where ``on average'' refers to an arbitrary probability measure that can depend on the test example.
These results are illustrated on the simple case of binary classification.
Section~\ref{sec:classification} gives applications to multi-class classification,
and Sect.~\ref{sec:conclusion} concludes listing some limitations of our results
and directions of further research.

Our notation for the base of natural logarithms will be $\e\approx2.72$
(while italic $e$ will often serve as a generic notation for e-values).

\section{Definitions}
\label{sec:definitions}

This paper deals with the following prediction problem.
We are given a training sequence of \emph{examples} $z_i=(x_i,y_i)$,
$i=1,\dots,n$ for a fixed $n$,
each consisting of an \emph{object} $x_i$ and its \emph{label} $y_i$,
and a new test object $x_{n+1}$;
the task is to predict $x_{n+1}$'s label $y_{n+1}$.
A potential label $y$ for $x_{n+1}$ is \emph{true} if $y=y_{n+1}$
and \emph{false} otherwise.
The objects are drawn from a non-empty measurable space $\mathbf{X}$,
the \emph{object space},
and the labels from the \emph{label space} $\mathbf{Y}$,
which is assumed to be a non-trivial measurable space
(meaning that the $\sigma$-algebra on it
is different from $\{\emptyset,\mathbf{Y}\}$).

A measurable function $P:\mathbf{Z}^{n+1}\to[0,1]$
is a \emph{randomness p-variable} if,
for any probability measure $Q$ on $\mathbf{Z}$
and any \emph{significance level} $\alpha\in(0,1)$,
$Q^{n+1}(\{P\le\alpha\})\le\alpha$.
Such a function $P$ is an \emph{exchangeability p-variable}
if $R(\{P\le\alpha\})\le\alpha$
for any exchangeable probability measure $R$ on $\mathbf{Z}^{n+1}$
and any $\alpha\in(0,1)$.
And a real-valued function $P$ defined on $\mathbf{Z}^{n+1}$
is \emph{train-invariant}
if it is invariant w.r.\ to permutations of the training examples:
\begin{equation*} % \label{eq:p-train-invariance}
  P(z_{\sigma(1)},\dots,z_{\sigma(n)},z_{n+1})
  =
  P(z_{1},\dots,z_n,z_{n+1})
\end{equation*}
for each data sequence $(z_1,\dots,z_{n+1})\in\mathbf{Z}^{n+1}$
and each permutation $\sigma$ of $\{1,\dots,n\}$.
In other words,
train-invariant functions should depend
on the training examples $z_1,\dots,z_n$ only via
the training bag $\lbag z_1,\dots,z_n\rbag$.
% train-invariant functions were called simply invariant
% in \cite{Vovk:arXiv2501}
Finally, $P$ is a \emph{conformal p-variable}
if it is a train-invariant exchangeability p-variable.

\begin{remark}
  The difference between the assumptions of randomness and exchangeability
  disappears for infinite data sequences
  under mild assumptions about the example space $\mathbf{Z}$
  (it is required to be Borel).
  This follows from de Finetti's theorem,
  which represents exchangeable probability measures
  as integral mixtures of product probability measures $Q^{\infty}$.
  The difference becomes essential for finite data sequences;
  see, e.g., \cite[Sections~2.9.1 and~A.5.1]{Vovk/etal:2022book}.
\end{remark}

We will sometimes refer to the values taken by p-variables
as \emph{p-values},
and our notation for the classes of all randomness,
train-invariant randomness, exchangeability, and conformal p-variables
will be $\PR$, $\PtR$, $\PX$, and $\PtX$, respectively.

Conformal p-variables can be used for prediction,
and we will also refer to them as \emph{conformal p-predictors}
(they are usually called simply ``conformal predictors''
\cite{Angelopoulos/etal:arXiv2411,Vovk/etal:2022book}).
There are several ways to package the output of conformal p-predictors,
as discussed in Sect.~\ref{sec:introduction}.
One is in terms of set prediction:
for each significance level $\alpha\in(0,1)$,
each training sequence $z_1,\dots,z_n$, and each test object $x_{n+1}$,
we can output the \emph{prediction set}
\begin{equation}\label{eq:Gamma}
  \Gamma^{\alpha}
  :=
  \{
    y\in\mathbf{Y}
    \mid
    P(z_1,\dots,z_n,x_{n+1},y)
    >
    \alpha
  \}.
\end{equation}
By the definition of a conformal p-variable,
under the assumption of exchangeability,
the probability that a conformal p-predictor makes an error
at significance level $\alpha$,
i.e., the probability of $y_{n+1}\notin\Gamma^{\alpha}$,
is at most $\alpha$.

Instead of predicting with one prediction set
in the family \eqref{eq:Gamma},
in this paper we prefer to package our prediction
as the \emph{prediction function}
\begin{equation}\label{eq:p-f}
  f(y)
  :=
  P(z_1,\dots,z_n,x_{n+1},y),
  \quad
  y\in\mathbf{Y}.
\end{equation}
We may refer to this mode of prediction as \emph{functional prediction}.
The step from set prediction to functional prediction
is analogous to the step from confidence intervals to p-value functions
(see, e.g., \cite[Sect.~9]{Miettinen:1985} and \cite{Fraser:2017,Fraser:2019}
% Infanger/Schmidt:2019
for the latter).

\begin{remark}\label{rem:equivalent}
  There are several equivalent definitions of conformal p-predictors,
  and the definition as train-invariant exchangeability p-variables
  (first given in \cite[Proposition~1]{Nouretdinov/etal:2003ALT})
  is one of them.
  Let us check that it is equivalent to, e.g., the definition of conformal p-values
  given in \cite[(2.20)]{Vovk/etal:2022book}
  for a fixed length $n$ of the training sequence.
  In one direction, it is obvious that conformal p-values are train-invariant
  and valid under exchangeability.
  On the other hand, given a train-invariant exchangeability p-variable $P$,
  we can define the nonconformity measure
  \[
    A(\lbag z_1,\dots,z_{n+1}\rbag,z_i)
    :=
    1/P(z_{i+1},\dots,z_{n+1},z_{1},\dots,z_i),
  \]
  and the resulting conformal p-values will dominate $P$
  (dominate in the sense of being less than or equal to).
\end{remark}

\begin{remark}
  The term ``functional prediction'' is a straightforward modification of ``set prediction''
  % (which is functional prediction with prediction functions being indicator functions)
  and ``p-value function'',
  but its disadvantage is that it is easy to confuse with function prediction,
  namely predicting a function
  (e.g., a biological function, such as that of a protein, or a mathematical function).
\end{remark}

Similarly, we can use randomness p-variables for prediction,
and then we refer to them as \emph{randomness p-predictors}.
By definition,
the probability that the prediction set \eqref{eq:Gamma}
derived from a randomness p-predictor
makes an error is at most $\alpha$,
this time under the assumption of randomness.
We will use the prediction functions \eqref{eq:p-f}
for randomness p-predictors as well.
Less important,
we call exchangeability p-variables exchangeability p-predictors.

Two standard desiderata for conformal, and by extension randomness, predictors
are their validity and efficiency.
In terms of the prediction function $f$,
validity concerns the value $f(y_{n+1})$ of $f$ at the true label
(the typical values should not be too small in p-prediction),
and efficiency concerns the values $f(y)$ at the false labels $y\ne y_{n+1}$
(they should be as small as possible in p-prediction).
Validity is automatic under randomness
(and even under exchangeability for conformal predictors),
and in this paper we are interested in the efficiency of conformal predictors
relative to other randomness predictors.
Later in the paper (Theorem~\ref{thm:Kolmogorov} and Corollary~\ref{cor:p-t} below)
% (Corollaries~\ref{cor:e-main} and~\ref{cor:p-main} below)
we will establish efficiency guarantees
for conformal prediction in terms of randomness prediction.

A nonnegative measurable function $E:\mathbf{Z}^{n+1}\to[0,\infty]$
is a \emph{randomness e-variable} if $\int E\dd Q^{n+1}\le1$
for any probability measure $Q$ on $\mathbf{Z}$.
It is an \emph{exchangeability e-variable} if $\int E\dd R\le1$
for any exchangeable probability measure $R$ on $\mathbf{Z}^{n+1}$.
We will denote the classes of all randomness and exchangeability e-variables
by $\ER$ and $\EX$, respectively.
The class of all measurable functions $E:\mathbf{Z}^{n+1}\to[0,\infty]$
is denoted by $\EEE$.
It is easy to see that $E\in\EEE$ belongs to $\EX$ if and only if,
for any data sequence $z_1,\dots,z_{n+1}$,
\begin{equation}\label{eq:EX}
  \frac{1}{(n+1)!}
  \sum_{\pi}
  E(z_{\pi(1)},\dots,z_{\pi(n+1)})
  \le
  1,
\end{equation}
$\pi$ ranging over the permutations of $\{1,\dots,n+1\}$.

The class $\EtX$ of \emph{conformal e-variables} consists of all functions $E\in\EX$
that are train-invariant.
We often regard the randomness e-variables $E\in\ER$
as \emph{randomness e-predictors}
and conformal e-variables $E\in\EtX$
as \emph{conformal e-predictors}.
Similarly to \eqref{eq:p-f},
they output prediction functions
\begin{equation*} % \label{eq:e-f}
  f(y)
  :=
  E(z_1,\dots,z_n,x_{n+1},y),
  \quad
  y\in\mathbf{Y}.
\end{equation*}

\begin{remark}
  Similarly to Remark \ref{rem:equivalent},
  it is easy to check that a train-invariant exchangeability e-variable
  is the same thing as a conformal e-predictor
  as defined in, e.g., \cite{Vovk:2025PR-local}.
  Indeed, the nonconformity e-measure \cite[Sect.~2]{Vovk:2025PR-local}
  corresponding to a train-invariant exchangeability e-variable $E$ is
  \begin{wideformula*}
    A(z_1,\dots,z_{n+1})
    :=
    \bigl(
      E(z_2,\dots,z_{n+1},z_1),
      E(z_3,\dots,z_{n+1},z_1,z_2),
      \dots,\NarrowLineBreak
      E(z_1,z_2,\dots,z_{n+1})
    \bigr).
  \end{wideformula*}
\end{remark}

The subclass $\EtR\subseteq\ER$ of all train-invariant randomness e-predictors
is important since under the assumption of randomness
it is natural to consider only train-invariant predictors:
the requirement of train-invariance
is a special case of the principle of sufficiency in statistical inference
\cite[2.3.(ii)]{Cox/Hinkley:1974}.
% and it is reflected in ``training set'' rather than ``training sequence''
% being used as standard term in machine learning.
Let us refer to the requirement of train-invariance under the assumption of randomness
as the \emph{train-invariance principle}.
The train-invariance principle is a special case not only of the sufficiency principle
but also of the invariance principle \cite[Example~2.35]{Cox/Hinkley:1974},
which makes it even more convincing.

For conformal and randomness e-predictors,
validity and efficiency change direction as compared with p-predictors:
for validity, typical values $f(y_{n+1})$ should not be too large,
and for efficiency typical values $f(y)$ at the false labels $y\ne y_{n+1}$
should be as large as possible.
Again validity is automatic under randomness,
and Theorem~\ref{thm:Kolmogorov} below establishes
efficiency guarantees for conformal e-prediction
in terms of train-invariant randomness e-prediction.
Similarly to $\ER$ and $\EtX$,
elements of the class $\EX$ will be referred to as \emph{exchangeability e-predictors},
but we are not particularly interested in them per se.
% it is intermediate between the randomness e-predictors and the conformal e-predictors,
% and it consists of randomness e-predictors that are valid under the assumption of exchangeability
% (which is weaker than the assumption of randomness in the case of finite data sequences).

We will also need another, even narrower, subclass of $\ER$, $\EiR\subseteq\EtR$.
The class $\EiR$ consists of all e-variables $E\in\ER$
that are invariant w.r.\ to all permutations:
\[
  E(z_{\pi(1)},\dots,z_{\pi(n+1)})
  =
  E(z_{1},\dots,z_{n+1})
\]
for each permutation $\pi$ of $\{1,\dots,n+1\}$;
let us call such randomness e-variables \emph{invariant}.
An equivalent definition is to say that
$E(z_{1},\dots,z_{n+1})$ depends on the data sequence $z_{1},\dots,z_{n+1}$
only via the bag $\lbag z_{1},\dots,z_{n+1}\rbag$ of its elements.
The interpretation is that the value $E(z_1,\dots,z_{n+1})$ for $E\in\EiR$
is the degree to which we can reject the hypothesis
that the bag $\lbag z_{1},\dots,z_{n+1}\rbag$
resulted from its elements being generated in the IID fashion.

A big advantage of e-variables over p-variables is that the average of e-variables
is again an e-variable.
This allows us to define, given an e-variable $E\in\ER$,
two derivative e-variables:
\begin{align}
  E\i(z_1,\dots,z_{n+1})
  &:=
  \frac{1}{(n+1)!}
  \sum_{\pi}
  E(z_{\pi(1)},\dots,z_{\pi(n+1)}),
  \label{eq:i}\\
  E\X(z_1,\dots,z_{n+1})
  &:=
  \frac{E(z_1,\dots,z_{n+1})}{E\i(z_1,\dots,z_{n+1})},
  \label{eq:X}
\end{align}
with $\pi$ ranging over the permutations of $\{1,\dots,n+1\}$
and $0/0$ interpreted as 1
(the last convention may be relevant to \eqref{eq:X}).
More generally, we can allow any $E\in\EEE$ in~\eqref{eq:i} and~\eqref{eq:X}.
It is clear that $E\i\in\EiR$ whenever $E\in\ER$
and that $E\X\in\EX$ for all $E\in\EEE$.
For $E\in\ER$,
the operator~\eqref{eq:i} performs averaging of e-values
over all permutations of an input data sequence,
and~\eqref{eq:X} is the relative deviation of an e-value from the average.

In the most important for us case
(transforming train-invariant randomness e-predictors to conformal e-predictors),
the operator \eqref{eq:X} is polynomially computable:
namely, if $E\in\EtR$ is efficiently computable,
the extra time when computing $E\X\in\EtX$
on top of computing $E$ is linear, $O(n)$.
This follows from
\[
  E\X(z_1,\dots,z_{n+1})
  =
  \frac
  {E(z_1,\dots,z_{n+1})}
  {
    \frac{1}{n+1}
    \sum_{i=1}^{n+1}
    % E(z_1,\dots,z_{i-1},z_{i+1},\dots,z_{n+1},z_i)
    E(z_{i+1},\dots,z_{n+1},z_{1},\dots,z_i)
  }
\]
for $E\in\EtR$.

\section{Main results}
\label{sec:results}

Let $B$ be a Markov kernel with source $\mathbf{Z}$ and target $\mathbf{Y}$,
which we will write in the form $B:\mathbf{Z}\hookrightarrow\mathbf{Y}$
(as in \cite[Sect.~A.4]{Vovk/etal:2022book}).
We will write $B(A\mid z)$ for its value on $z\in\mathbf{Z}$ and $A\subseteq\mathbf{Y}$
(where $A$ is measurable)
and write $\int f(y)B(\d y\mid z)$ for the integral of a function $f$ on $\mathbf{Y}$
w.r.\ to the measure $A\mapsto B(A\mid z)$.
We will show that the efficiency of the conformal predictor
derived from a train-invariant randomness predictor $E$
is not much worse than the efficiency of the original randomness predictor $E$
on average,
and $B$ will define the meaning of ``on average''.

\subsection{From train-invariant randomness e-prediction to conformal e-prediction}
\label{subsec:e}

The following statement shows that the efficiency
does not suffer much on average when we move from randomness e-prediction
to exchangeability e-prediction.
It is more general than what we need at the moment,
since it also covers randomness e-predictors that are not train-invariant.

\begin{theorem}\label{thm:Kolmogorov}
  Let $B:\mathbf{Z}\hookrightarrow\mathbf{Y}$ be a Markov kernel.
  For each randomness e\-/predictor $E$,
  \begin{equation}\label{eq:Kolmogorov}
    G(z_1,\dots,z_n,z_{n+1})
    :=
    \e^{-1}
    \int
    \frac
      {E(z_1,\dots,z_n,x_{n+1},y)}
      {E\X(z_1,\dots,z_n,x_{n+1},y)}
    B(\d y\mid z_{n+1})
  \end{equation}
  (with $0/0$ set to 0 and $z_{n+1}$ represented as $(x_{n+1},y_{n+1})$)
  is a randomness e-variable.
\end{theorem}

We can interpret \eqref{eq:Kolmogorov} as a statement
that $E\X$ is almost as efficient as $E$
unless the true data sequence $z_1,\dots,z_{n+1}$ does not look IID.
The ``unless'' clause makes sense in view of our assumption of randomness.
According to \eqref{eq:Kolmogorov}
and assuming that $B(\cdot\mid z_{n+1})$
is concentrated on $\mathbf{Y}\setminus\{y_{n+1}\}$,
the mean ratio of the degree to which $E$ rejects a false label $y$
to the degree to which $E\X$ rejects $y$
is not large under any probability measure
that may depend on the test example
unless we can reject the assumption of randomness for the true data sequence.
The gap $\e^{-1}\approx0.37$ between the mean ratio
and the e-value at which we reject the assumption of randomness
is optimal (Theorem~\ref{thm:anti-Kolmogorov} below).

From now on, until the end of the main paper,
let us assume that $E$ is train-invariant in Theorem~\ref{thm:Kolmogorov},
i.e., $E$ is a train-invariant randomness e-predictor
(the only exception is the line following the statement of Corollary~\ref{cor:p-t}).
Then $E\X$ in \eqref{eq:Kolmogorov} is a conformal e-predictor.

A full proof of Theorem~\ref{thm:Kolmogorov}
will be given in Sect.~\ref{subapp:proof-Kolmogorov},
but we will demonstrate the idea of the proof on a simple special case,
which also makes the statement of the theorem more tangible.
In the case of binary classification, $\mathbf{Y}:=\{-1,1\}$,
the most natural choice of $B$ is $B(\{y\}\mid(x,y)):=0$,
so that the Markov kernel sends every example $(x,y)$ to the other label $-y$.
We can then rewrite \eqref{eq:Kolmogorov} as
\begin{equation}\label{eq:binary}
  G(z_1,\dots,z_n,z_{n+1})
  :=
  \e^{-1}
  \frac
    {E(z_1,\dots,z_n,x_{n+1},-y_{n+1})}
    {E\X(z_1,\dots,z_n,x_{n+1},-y_{n+1})},
\end{equation}
which does not involve any averaging.
We can interpret~\eqref{eq:binary} as the conformal e-predictor $E\X$
being almost as efficient as the original train-invariant randomness e-predictor $E$,
where efficiency is measured by the degree to which we reject the false label $-y_{n+1}$.
For example, for a small positive constant $\epsilon$,
$G\ge1/\epsilon$ with probability at most $\epsilon$,
and so
\begin{equation}\label{eq:epsilon}
  E\X(z_1,\dots,z_n,x_{n+1},-y_{n+1})
  >
  \e^{-1} \epsilon E(z_1,\dots,z_n,x_{n+1},-y_{n+1})
\end{equation}
with probability at least $1-\epsilon$.

To get an idea of the proof,
suppose the true data sequence $z_1,\dots,z_{n+1}$,
generated in the IID fashion,
is such that the ratio in \eqref{eq:binary} is large.
By definition, the ratio $E\i=E/E\X$ is an invariant randomness e-variable,
and so only depends on the bag of its input examples
and measures the degree to which that bag does not look IID.
Flipping the label of a randomly chosen $z_i$, $i\in\{1,\dots,n+1\}$,
in the training bag $\lbag z_1,\dots,z_{n+1}\rbag$
leads to a data bag that is still compatible with the assumption of randomness,
and if flipping the label $y_{n+1}$ in the training bag
leads to a large value of $E/E\X$,
this means that the example $z_{n+1}$ was unusual in the training bag,
which makes the original data sequence $z_1,\dots,z_{n+1}$,
in which the unusual example is also the last one,
not compatible with the assumption of randomness.

The following result is a simple statement of optimality for Theorem~\ref{thm:Kolmogorov}.

\begin{theorem}\label{thm:anti-Kolmogorov}
  The constant $\e^{-1}$ in Theorem~\ref{thm:Kolmogorov}
  cannot be replaced by a larger one,
  even if we assume $E\in\EtR$.
\end{theorem}

In principle, Theorem~\ref{thm:anti-Kolmogorov}
follows from a later result (namely, Theorem~\ref{thm:lower-1} below),
but in Sect.~\ref{subapp:proof-anti-Kolmogorov}
we will give a simple independent proof.
The origin of the factor $\e^{-1}$
is the difference between the assumptions of randomness and exchangeability:
while flipping the label of a randomly chosen example $z_i$, $i\in\{1,\dots,n+1\}$,
keeps the exchangeability of the original IID data sequence,
the new sequence ceases to be generated in the IID fashion.
Therefore, we approximate flipping the label of one randomly chosen example
by flipping the label of each example $z_i$, $i\in\{1,\dots,n+1\}$,
with a small probability;
$\e^{-1}$ is the largest probability (which is attainable)
that exactly one label will be flipped.
Full details are given in Appendix~\ref{app:proofs}.

\subsection{From train-invariant randomness p-prediction to conformal p-prediction}
\label{subsec:p}

It is known that, for any $\delta\in(0,1)$,
the function $p\mapsto\delta p^{\delta-1}$
transforms p-values to e-values
and that the function $e\mapsto e^{-1}$
transforms e-values to p-values.
See, e.g., \cite[Propositions~2.1 and~2,2]{Vovk/Wang:2021}.
More generally, any function $f:[0,1]\to[0,\infty]$ integrating to 1
transforms p-values to e-values
(is a \emph{p-to-e calibrator}),
while $e\mapsto\min(e^{-1},1)$ is the optimal way of transforming
e-values to p-values
(it is an optimal \emph{e-to-p calibrator}).
This allows us to adapt Theorem~\ref{thm:Kolmogorov} to p-predictors.

\begin{corollary}\label{cor:p-t}
  Let $B:\mathbf{Z}\hookrightarrow\mathbf{Y}$ be a Markov kernel
  and let $\delta\in(0,1)$.
  For each randomness p-predictor $P$
  there exists an exchangeability p-predictor $P'$ such that
  \begin{equation}\label{eq:cor-p-t}
    G(z_1,\dots,z_n,z_{n+1})
    :=
    \frac{\delta}{\e}
    \int
      \frac
        {P'(z_1,\dots,z_n,x_{n+1},y)}
        {P^{1-\delta}(z_1,\dots,z_n,x_{n+1},y)}\;
    B(\d y\mid z_{n+1})
  \end{equation}
  is a randomness e-variable.
\end{corollary}

The proof is obvious
(calibrate $P$ to get $E\in\EX$ and then calibrate $E\X$ to get $P'\in\PX$),
but it is still spelled out in Sect.~\ref{subapp:proof-p-t}.
As before, we concentrate on the case where $P$ is train-invariant,
and in this case $P'$ can be chosen as conformal p-predictor.

The interpretation of \eqref{eq:cor-p-t} is much simpler
in the binary case $\mathbf{Y}=\{-1,1\}$ with the same Markov kernel as before.
In this case \eqref{eq:cor-p-t} becomes
\begin{equation*}
  G(z_1,\dots,z_n,z_{n+1})
  :=
  \frac{\delta}{\e}
  \frac
    {P'(z_1,\dots,z_n,x_{n+1},-y_{n+1})}
    {P^{1-\delta}(z_1,\dots,z_n,x_{n+1},-y_{n+1})}.
\end{equation*}
Therefore,
\begin{equation*}
  P'(z_1,\dots,z_n,x_{n+1},-y_{n+1})
  <
  \frac{\e}{\delta\epsilon}
  P^{1-\delta}(z_1,\dots,z_n,x_{n+1},-y_{n+1})
\end{equation*}
with probability at least $1-\epsilon$.

In conclusion of this section,
let us discuss \eqref{eq:cor-p-t} in general.
The interpretation of \eqref{eq:cor-p-t} is that,
under the randomness of the true data sequence,
$P'(z_1,\dots,z_n,x_{n+1},y)$ is typically small
(perhaps not to the same degree)
when $P(z_1,\dots,z_n,x_{n+1},y)$ is small;
i.e., we do not lose much in efficiency
when converting train\-/invariant randomness p-predictors to conformal p-predictors.
To see this, fix small $\epsilon_1,\epsilon_2\in(0,1)$.
Then we will have $G(z_1,\dots,z_n,z_{n+1})<1/\epsilon_1$
for the true data sequence $z_1,\dots,z_n,z_{n+1}$
unless a rare event (of probability at most $\epsilon_1$) happens.
For the vast majority of the potential labels $y\in\mathbf{Y}$ we will then have
\begin{equation}\label{eq:bounded-t}
  \frac{\delta}{\e}
  \frac
    {P'(z_1,\dots,z_n,x_{n+1},y)}
    {P^{1-\delta}(z_1,\dots,z_n,x_{n+1},y)}
  <
  \frac{1}{\epsilon_1\epsilon_2},
\end{equation}
where ``the vast majority'' means that the $B(\cdot\mid z_{n+1})$ measure
of the $y$ satisfying \eqref{eq:bounded-t} is at least $1-\epsilon_2$.
We can rewrite \eqref{eq:bounded-t} as
\begin{equation}\label{eq:rewrite}
  P'(z_1,\dots,z_n,x_{n+1},y)
  <
  \frac{\e}{\delta\epsilon_1\epsilon_2}
  P^{1-\delta}(z_1,\dots,z_n,x_{n+1},y),
\end{equation}
so that $P'(z_1,\dots,z_n,x_{n+1},y)\to0$ as $P(z_1,\dots,z_n,x_{n+1},y)\to0$.
This is, of course, true for any Markov kernel $B$.

\section{Applications to multi-class classification and optimality results}
\label{sec:classification}

In this section we discuss the case of classification, $\left|\mathbf{Y}\right|<\infty$,
but now we are interested in the non-binary case $\left|\mathbf{Y}\right|>2$
(the $\sigma$-algebra on $\mathbf{Y}$ is discrete, as usual).
% which was also discussed earlier in \cite{Vovk:arXiv2501}.
Let us only discuss reduction of randomness e-predictors to conformal e-predictors.
Reduction of randomness p-predictors to conformal p-predictors is completely analogous;
it just uses \eqref{eq:cor-p-t} instead of \eqref{eq:Kolmogorov}.

In the case of multi-class classification, $2<\left|\mathbf{Y}\right|<\infty$,
the most natural Markov kernel $B$ is perhaps the one for which $B(\cdot\mid(x,y))$
is the uniform probability measure on $\mathbf{Y}\setminus\{y\}$.
In this case we can rewrite \eqref{eq:Kolmogorov} as
\begin{equation}\label{eq:multi-class}
  G(z_1,\dots,z_n,z_{n+1})
  :=
  \frac{\e^{-1}}{\left|\mathbf{Y}\right|-1}
  \sum_{y\in\mathbf{Y}\setminus\{y_{n+1}\}}
  \frac
    {E(z_1,\dots,z_n,x_{n+1},y)}
    {E\X(z_1,\dots,z_n,x_{n+1},y)}.
\end{equation}
As before, we assume that $E$ is train-invariant.
The interpretation of~\eqref{eq:multi-class} is that the conformal e-predictor $E\X$
is almost as efficient as the original train-invariant randomness e-predictor $E$ on average;
as before, efficiency is measured by the degree
to which we reject the false labels $y\ne y_{n+1}$.
Roughly, on average, we lose at most a factor of $\e$ in the e-values of false labels
when we replace $E$ by $E\X$.

Of course, we can avoid ``on average'' by making \eqref{eq:multi-class} cruder
and replacing it by the existence of $G\in\ER$ satisfying
\begin{multline}\label{eq:anti-Kolmogorov-1}
  \forall(z_1,\dots,z_n)\in\mathbf{Z}^n \;
  \forall x_{n+1}\in\mathbf{X} \;
  \forall y\in\mathbf{Y}\setminus\{y_{n+1}\}:\\
  G(z_1,\dots,z_n,z_{n+1})
  \ge
  \frac{\e^{-1}}{\left|\mathbf{Y}\right|-1}
  \frac
    {E(z_1,\dots,z_n,x_{n+1},y)}
    {E\X(z_1,\dots,z_n,x_{n+1},y)},
\end{multline}
where $z_{n+1}:=(x_{n+1},y_{n+1})$.
For a small positive constant $\epsilon$,
we can then claim that,
with probability at least $1-\epsilon$ over the true data sequence,
\begin{equation}\label{eq:epsilon-worst}
  \forall y\in\mathbf{Y}\setminus\{y_{n+1}\}:
  E\X(z_1,\dots,z_n,x_{n+1},y)
  >
  \frac{\e^{-1}\epsilon}{\left|\mathbf{Y}\right|-1}
  E(z_1,\dots,z_n,x_{n+1},y).
\end{equation}

An interesting variation of \eqref{eq:multi-class},
corresponding to the Markov kernel $B$
for which $B(\cdot\mid(x,y))$ is the uniform probability measure on $\mathbf{Y}$,
is
\begin{equation}\label{eq:anti-Kolmogorov-2}
  G(z_1,\dots,z_n,z_{n+1})
  :=
  \frac{\e^{-1}}{\left|\mathbf{Y}\right|}
  \sum_{y\in\mathbf{Y}}
  \frac
    {E(z_1,\dots,z_n,x_{n+1},y)}
    {E\X(z_1,\dots,z_n,x_{n+1},y)}.
\end{equation}
Under this definition,
the randomness e-variable $G$ does not depend on $y_{n+1}$.

Let us check that the denominators,
$\left|\mathbf{Y}\right|-1$ or $\left|\mathbf{Y}\right|$,
in \eqref{eq:multi-class}, \eqref{eq:anti-Kolmogorov-1}, and \eqref{eq:anti-Kolmogorov-2}
are asymptotically optimal.

\begin{theorem}\label{thm:lower-1}
  For each constant $c>1$ the following statement holds true
  for a sufficiently large $\left|\mathbf{Y}\right|$ and a sufficiently large $n$.
  There exists a train-invariant randomness e-predictor $E$
  such that for each randomness e-variable $G$
  there exist $z_1,\dots,z_n$, $z_{n+1}=(x_{n+1},y_{n+1})$, and $y\ne y_{n+1}$
  such that
  \begin{equation*} % \label{eq:lower-1}
    G(z_1,\dots,z_n,z_{n+1})
    <
    \frac{c\e^{-1}}{\left|\mathbf{Y}\right|}
    \frac
      {E(z_1,\dots,z_n,x_{n+1},y)}
      {E\X(z_1,\dots,z_n,x_{n+1},y)}.
  \end{equation*}
\end{theorem}

\noindent
When we say ``for a sufficiently large $\left|\mathbf{Y}\right|$''
in Theorem~\ref{thm:lower-1},
the lower bound on $\left|\mathbf{Y}\right|$ is allowed to depend on $c$,
and when we say ``and a sufficiently large $n$'',
the lower bound on $n$ is allowed to depend on $c$ and $\left|\mathbf{Y}\right|$.

A complete proof of Theorem~\ref{thm:lower-1}
is given in Sect.~\ref{subapp:proof-lower-1},
but the informal idea of the proof is that we can ignore the objects
and make a false test label $y\ne y_{n+1}$
encode the bag $\lbag y_1,\dots,y_n\rbag$ of the training labels.
We have $E\i:=E/E\X\in\EiR$.
If we also make $y$ easily distinguishable
from the training labels $y_1,\dots,y_n$,
the value $E\i(z_1,\dots,z_n,x_{n+1},y)$
(even if depending only on $\lbag z_1,\dots,z_n,x_{n+1},y\rbag$)
can be made large,
even for a true data sequence generated in the IID fashion.

Formally, Theorem~\ref{thm:lower-1} is an inverse
to \eqref{eq:anti-Kolmogorov-1} and \eqref{eq:anti-Kolmogorov-2},
but it has several weaknesses in this role:
\begin{itemize}
\item
  It shows that $E\X$ is not competitive with $E$ in some cases,
  but can there be another conformal e-predictor $E'\in\EtX$
  that is better than $E\X$ in this respect?
\item
  Even if the ratio $E(z_1,\dots,z_n,x_{n+1},y)/E'(z_1,\dots,z_n,x_{n+1},y)$
  for $E'\in\EtX$
  is very large,
  it is not so interesting if already $E'(z_1,\dots,z_n,x_{n+1},y)$
  is very large.
\item
  Even if the ratio $E(z_1,\dots,z_n,x_{n+1},y)/E'(z_1,\dots,z_n,x_{n+1},y)$
  for $E'\in\EtX$
  is very large,
  it is not so interesting if $G(z_1,\dots,z_n,z_{n+1})$
  is very large for the true data sequence.
\end{itemize}
The following result overcomes these weaknesses.

\begin{theorem}\label{thm:lower-2}
  For each constant $c\in(0,1)$ and each label space $\mathbf{Y}$,
  $1<\left|\mathbf{Y}\right|<\infty$,
  the following statement holds true for
  sufficiently large $n$.
  There exists a train-invariant randomness e-predictor $E$
  such that for each conformal e-predictor $E'$ and each randomness e-variable $G$
  there exist $z_1,\dots,z_n$, $z_{n+1}=(x_{n+1},y_{n+1})$, and $y\ne y_{n+1}$
  such that
  \begin{align}
    G(z_1,\dots,z_n,z_{n+1})
    &\le
    2,\label{eq:system-G}\\
    E'(z_1,\dots,z_n,x_{n+1},y)
    &\le
    2.01,\label{eq:system-E'-used}\\
    E(z_1,\dots,z_n,x_{n+1},y)
    &\ge
    c\left|\mathbf{Y}\right|.\label{eq:system-E}
  \end{align}
\end{theorem}

To compare Theorem~\ref{thm:lower-2} with Theorem~\ref{thm:lower-1},
notice that \eqref{eq:system-G}--\eqref{eq:system-E} imply
\begin{equation}
  \notag
  G(z_1,\dots,z_n,z_{n+1})
  <
  \frac{4.02 c^{-1}}{\left|\mathbf{Y}\right|}
  \frac{E(z_1,\dots,z_n,x_{n+1},y)}{E'(z_1,\dots,z_n,x_{n+1},y)}.
\end{equation}
Therefore, Theorem~\ref{thm:lower-2} implies Theorem~\ref{thm:lower-1},
but only if we ignore the constant factor $4.02\e<11$.

\begin{proof}[Proof sketch of Theorem~\ref{thm:lower-2}]
  Let us ignore the objects and set,
  without loss of generality (assuming our prediction problem is classification),
  $\mathbf{Y}:=\{0,\dots,\left|\mathbf{Y}\right|-1\}$.
  Generate labels $Y_1,\dots,Y_{n+1}$ randomly
  (independently from the uniform distribution on $\mathbf{Y}$),
  and set $Y\equiv-Y_1-\dots-Y_n \Mod{\left|\mathbf{Y}\right|}$;
  capital letters are used to emphasize that the labels (elements of $\mathbf{Y}$)
  are random.
  The idea is to prove that the three events
  \begin{align}
    G(Y_1,\dots,Y_n,Y_{n+1})
    &\le
    2,
    \label{eq:Y-G}\\
    E'(Y_1,\dots,Y_n,Y)
    &\le
    2.01,
    \label{eq:Y-E'}\\
    E(Y_1,\dots,Y_n,Y)
    &\ge
    c\left|\mathbf{Y}\right|
    \label{eq:Y-E}
  \end{align}
  (cf.\ \eqref{eq:system-G}--\eqref{eq:system-E}, respectively)
  hold with probability at least 0.5, 0.502, and 0.999, respectively.
  This will imply the statement of the theorem as their intersection
  will be nonempty.

  Since $Y_1,\dots,Y_{n+1}$ are IID and $Y_1,\dots,Y_n,Y$ are exchangeable,
  the probabilities of \eqref{eq:Y-G} and \eqref{eq:Y-E'} are bounded
  by Markov's inequality.
  And we can define $E$ to ensure that \eqref{eq:Y-E} holds with a small probability
  since $Y_1,\dots,Y_n,Y$ are not IID:
  were they IID,
  we would typically expect their sum to be divisible by $\left|\mathbf{Y}\right|$
  with probability close to $1/\left|\mathbf{Y}\right|$ for a large $n$.

  Further details are given in Sect.~\ref{subapp:proof-lower-2}.
\end{proof}

\section{Conclusion}
\label{sec:conclusion}

This paper gives explicit statements,
not involving any unspecified constants,
of universality of conformal predictors.
Namely, for each train-invariant randomness predictor
there is a conformal predictor that is competitive with it.
Some constants is these statements have been shown to be optimal.

These are some directions of further research:
\begin{itemize}
\item
  This paper shows that the attainable improvement on conformal prediction
  under the assumption of randomness is limited,
  such as a factor of $\e$ in the e-values for false labels
  (see, e.g., \eqref{eq:binary}).
  Can we develop practically useful predictors
  exploiting such potential improvements?
  (For a toy example, see Remark~\ref{rem:toy}
  in Sect.~\ref{subapp:proof-anti-Kolmogorov}.)
\item
  Can we connect $\PtR$ and $\PtX$ directly
  (in the spirit of Corollary~\ref{cor:p-t}),
  without a detour via e-values?
\item
  A related remark is that all our optimality results
  (Theorems~\ref{thm:anti-Kolmogorov}, \ref{thm:lower-1}, and~\ref{thm:lower-2})
  cover only e-prediction.
  It would be ideal to have similar optimality results
  for Corollary~\ref{cor:p-t} or its stronger versions.
\item
  The assumption of randomness is very strong,
  and there has been extensive work devoted to relaxing this assumption;
  see, e.g., \cite{Tibshirani/etal:2019} and \cite[Chap.~7]{Angelopoulos/etal:arXiv2411}.
  To what degree do the results of this paper carry over to more general settings?
\end{itemize}

\iftoggle{TR}{%
  \subsection*{Acknowledgments}

  In October 2024 Vladimir V'yugin,
  one of my teachers in the algorithmic theory of randomness,
  died at the age of 75.
  I am deeply grateful to him for his support and encouragement
  from the time when I first met him, in 1980, until the end of his life.
  The research reported in this paper is dedicated to his memory.%
}{}%

\appendix
\section{Proofs}
\label{app:proofs}

This appendix gives detailed proofs of all results stated in the main paper.
The \emph{Bernoulli model} is defined as the statistical model
$(B_{\theta}^{n+1}:\theta\in[0,1])$,
where $B_{\theta}$ is the \emph{Bernoulli measure} on $\{0,1\}$,
defined by $B_{\theta}(\{1\}):=\theta\in[0,1]$.

\subsection{Proof of Theorem~\ref{thm:Kolmogorov}}
\label{subapp:proof-Kolmogorov}

% \begin{proof}[Proof of Theorem~\ref{thm:Kolmogorov}]
  We will define $G$ as $G_2G_3$,
  where $G_2\in\EiR$ and $G_3\in\EX$
  (it is obvious that these two inclusions will imply $G\in\ER$).
  First we define an approximation $G_1$ to $G_2$ as
  \[
    G_1(z_1,\dots,z_{n+1})
    :=
    \frac{1}{n+1}
    \sum_{i=1}^{n+1}
    \int E\i(z_1,\dots,z_{i-1},x_i,y,z_{i+1},\dots,z_{n+1})
    B(\dd y\mid z_i).
  \]
  In other words, $G_1(z_1,\dots,z_{n+1})$ is obtained
  by randomly (with equal probabilities)
  choosing an example $z_i$ in the data sequence $z_1,\dots,z_{n+1}$,
  replacing its label $y_i$ by a random label $y\sim B(\cdot\mid z_i)$,
  and finding the expectation of $E\i$ on $z_1,\dots,z_{n+1}$ modified in this way.
  We can see that $G_1$ is invariant,
  but it does not have to be in $\EiR$.
  The invariant randomness e-variable $G_2$ is defined similarly,
  except that now we replace each label $y_i$, $i=1,\dots,n+1$,
  by a random label $y\sim B(\cdot\mid z_i)$
  with probability $\frac{1}{n+1}$ (all independently).
  The key observation is that $G_2/G_1\ge1/\e$,
  which follows from the probability that exactly one label will be changed
  in the construction of $G_2$ being
  \[
    (n+1)
    \frac{1}{n+1}
    \left(\frac{n}{n+1}\right)^n
    \ge
    1/\e.
  \]
  Finally, $G_3\in\EX$ is defined by
  \[
    G_3(z_1,\dots,z_{n+1})
    :=
    \frac
      {\int E\i(z_1,\dots,z_n,x_{n+1},y) B(\d y\mid z_{n+1})}
      {G_1(z_1,\dots,z_{n+1})}.
  \]
  Combining all these statements,
  we get the following randomness e-variable $G'$:
  \begin{align}
    G'(z_1,\dots,z_{n+1})
    &:=
    G_2(z_1,\dots,z_{n+1})
    G_3(z_1,\dots,z_{n+1})\notag\\
    &\ge
    \e^{-1}
    G_1(z_1,\dots,z_{n+1})
    G_3(z_1,\dots,z_{n+1})\notag\\
    &=
    \e^{-1}
    \int E\i(z_1,\dots,z_n,x_{n+1},y) B(\d y\mid z_{n+1}).
    \label{eq:Kolmogorov-chain}
  \end{align}
  By the definition \eqref{eq:X},
  \eqref{eq:Kolmogorov} is equal to the expression in \eqref{eq:Kolmogorov-chain},
  and so $G'\in\ER$ implies that \eqref{eq:Kolmogorov} is also a randomness e-variable.
% \end{proof}

\subsection{Proof of Theorem~\ref{thm:anti-Kolmogorov}}
\label{subapp:proof-anti-Kolmogorov}

% \begin{proof}
  Without loss of generality we assume $\left|\mathbf{X}\right|=1$
  (so that the objects become uninformative and we can omit them from our notation)
  and $\mathbf{Y}=\{-1,1\}$ (with the discrete $\sigma$-algebra).
  Define a randomness e-variable $E$ by
  \begin{equation}\label{eq:E-simple}
    E(y_1,\dots,y_{n+1})
    :=
    \begin{cases}
      \left(
        1-\frac{1}{n+1}
      \right)^{-n} & \text{if $k=1$}\\
      0 & \text{if not},
    \end{cases}
  \end{equation}
  where $k$ is the number of 1s among $y_1,\dots,y_{n+1}$.
  This is indeed a randomness e-variable,
  since the maximum probability of $k=1$ in the Bernoulli model,
  $(n+1)p(1-p)^n\to\max$,
  is attained at $p=\frac{1}{n+1}$.
  The corresponding exchangeability e-variable is
  \[
    E\X(y_1,\dots,y_{n+1})
    =
    \begin{cases}
      1 & \text{if $k=1$}\\
      0 & \text{if not}.
    \end{cases}.
  \]
  Both $E$ and $E\X$ are train-invariant (and even invariant).
  Let $B$ just flip the label: $B(\{-y\}\mid y)=1$.
  Suppose Theorem~\ref{thm:Kolmogorov} holds with the $\e^{-1}$ in \eqref{eq:Kolmogorov}
  replaced by $c>\e^{-1}$.
  Then the randomness e-variable $G$ satisfies
  \[
    G(0,\dots,0)
    =
    c
    \left(
      1-\frac{1}{n+1}
    \right)^{-n}
    \sim
    c\e
    >
    1,
  \]
  which is impossible for a large enough $n$
  (since the probability measure concentrated on $(0,\dots,0)$
  is of the form $Q^{n+1}$).
% \end{proof}

\begin{remark}\label{rem:toy}
  Whereas the randomness e-variable $E$ defined by \eqref{eq:E-simple}
  is all we need to prove Theorem~\ref{thm:anti-Kolmogorov},
  it is not useful for prediction.
  A variation on \eqref{eq:E-simple} that can be used in prediction is
  \begin{equation*}
    E(y_1,\dots,y_{n+1})
    :=
    \begin{cases}
      (n+1)
      \left(
        1-\frac{1}{n+1}
      \right)^{-n} & \text{if $(y_1,\dots,y_n,y_{n+1})=(0,\dots,0,1)$}\\
      0 & \text{if not}.
    \end{cases}
  \end{equation*}
  According to this randomness e-predictor,
  after observing $n$ 0s in a row,
  we are likely to see 0 rather than 1.
  This is a version of Laplace's rule of succession.
  While under randomness we have $E(0,\dots,0,1)\sim\e n$,
  under exchangeability we can only achieve $E\X(0,\dots,0,1)=n+1\sim n$.
\end{remark}

\subsection{Proof of Corollary~\ref{cor:p-t}}
\label{subapp:proof-p-t}

% \begin{proof}[Proof of Corollary~\ref{cor:p-t}]
  Fix $\delta\in(0,1)$ and $P\in\PR$.
  Set $E:=\delta P^{\delta-1}$ and $P':=1/E\X$,
  so that $E\in\ER$ and $P'\in\PX$.
  According to \eqref{eq:Kolmogorov},
  \begin{align*}
    G(z_1,\dots,z_n,z_{n+1})
    &:=
    \e^{-1}
    \int
    \frac
      {E(z_1,\dots,z_n,x_{n+1},y)}
      {E\X(z_1,\dots,z_n,x_{n+1},y)}\;
    B(\d y\mid z_{n+1})\\
    &=
    \e^{-1}
    \int
    \frac
      {\delta P'(z_1,\dots,z_n,x_{n+1},y)}
      {P^{1-\delta}(z_1,\dots,z_n,x_{n+1},y)}\;
    B(\d y\mid z_{n+1})
  \end{align*}
  is a randomness e-variable.
% \end{proof}

\subsection{Proof of Theorem~\ref{thm:lower-1}}
\label{subapp:proof-lower-1}

% \begin{proof}
  The examples $z_1,\dots,z_{n+1}$
  whose existence is asserted in the statement of the theorem
  and which will be constructed in this proof
  will all share the same fixed object $x_0\in\mathbf{X}$,
  and we will sometimes omit ``$x_0$'' in our notation.
  Suppose, without loss of generality,
  that the label set $\mathbf{Y}$ is the disjoint union of $\{0,1\}$
  and the set $\{-k,\dots,k\}$ for some positive integer $k$;
  to distinguish between the 0s and the 1s in these two disjoint sets,
  let us write $0'$ and $1'$ for the elements of the first set,
  $\{0,1\}=\{0',1'\}$.
  (The primes are ignored, of course,
  when $0'$ and $1'$ are used as inputs to arithmetic operations,
  as in \eqref{eq:y} below.
  If $\left|\mathbf{Y}\right|$ is an even number,
  we can leave one of its elements unused.)
  To avoid trivial complications, let $n$ be an even number.
  This proof will assume $1\ll k\ll\sqrt{n}$;
  the formal meaning of this assumption will be summarized at the end of the proof.

  Define $E\in\EtR$ (which ignores the objects) as follows:
  \begin{itemize}
  \item
    on the sequences in $\mathbf{Z}^{n+1}$ of the form
    $(x_1,y_1,\dots,x_n,y_n,x_{n+1},y)$,
    where $y_1,\dots,y_n\in\{0',1'\}$, $y\in\{-k,\dots,k\}$, and
    \begin{equation}\label{eq:y}
      \sum_{i=1}^n y_i - n/2 = y,
    \end{equation}
    $E$ takes value $a\e\sqrt{\pi/2}n^{3/2}$,
    where $a<1$ is a positive constant
    (it will be taken close to 1 later in the proof);
  \item
    $E$ takes value 0 on all other sequences in $\mathbf{Z}^{n+1}$.
  \end{itemize}
  In \eqref{eq:y}, $y$ is determined by $y_1+\dots+y_n$ and vice versa.
  This agrees with the idea of the proof of Theorem~\ref{thm:lower-1}
  given after its statement:
  since $y_1,\dots,y_n$ are binary,
  $y_1+\dots+y_n$ carries the same information as $\lbag y_1,\dots,y_n\rbag$.

  Let us check that $E$ is indeed a train-invariant randomness e-variable.
  Let the underlying probability space be $\mathbf{Z}^{n+1}$
  equipped with a probability measure $R=Q^{n+1}$,
  so that individual examples are generated independently from $Q$.
  We will use the notation $Z_i$, $i=1,\dots,n+1$, for $z_i$ considered as a random example
  (formally, $Z_i$ is the random element defined by $Z_i(z_1,\dots,z_{n+1}):=z_i$)
  and the notation $Y_i$, $i=1,\dots,n+1$, for the label of $Z_i$.
  The maximum probability of the event $E(Z_1,\dots,Z_{n+1})>0$
  is attained for $Q$ giving maximum probabilities to the following two events:
  \begin{enumerate}
  \item\label{it:max-1}
    The random labels $Y_1,\dots,Y_{n}$ take values in $\{0',1'\}$,
    and the remaining random label $Y_{n+1}$ takes value in $\{-k,\dots,k\}$.
  \item\label{it:max-2}
    Conditionally on the first event, we have \eqref{eq:y},
    where $y$ is the value taken by $Y_{n+1}$
    and $y_1,\dots,y_n$ are the values taken by $Y_1,\dots,Y_n$.
  \end{enumerate}
  The maximum probability of the first event (in item~\ref{it:max-1}) is
  \begin{equation*} % \label{eq:e}
    % \binom{n+1}{1}
    \left(
      1-\frac{1}{n+1}
    \right)^n
    \frac{1}{n+1}
    \sim
    \frac{1}{\e n},
  \end{equation*}
  which is obtained by maximizing $(1-\theta)^n\theta$ over $\theta$
  and where the asymptotic equivalence is as $n\to\infty$.
  The maximum probability of the second event (in item~\ref{it:max-2})
  conditional on both the first event and $Y_{n+1}=y\in\{-k,\dots,k\}$ 
  is asymptotically equivalent,
  by the local limit theorem \cite[Sect.~1.6]{Shiryaev:2016}, to
  \[
    \frac{1}{\sqrt{2\pi n\left(\frac12+\frac{y}{n}\right)\left(\frac12-\frac{y}{n}\right)}}
    =
    \frac{1}{\sqrt{2\pi n\left(\frac14-\frac{y^2}{n^2}\right)}}
    \sim
    \sqrt{\frac{2}{\pi n}};
  \]
  this follows from the random variables $Y_1,\dots,Y_{n}$
  being distributed as $B_{\theta}^n$,
  and the maximum over $\theta$ being attained at $\theta=1/2+y/n$.
  Since $a<1$, $E\in\EtR$ for a sufficiently large $n$.
  Notice that our argument in this paragraph only uses $k\ll n$,
  since in our application of the local limit theorem
  the exponential term $\exp(\dots)$ was~1.

  Let us now generate $Y_1,\dots,Y_{n+1}$ from the Bernoulli model
  and find the maximum probability that
  \begin{equation}\label{eq:condition}
    \sum_{i=1}^n Y_i - n/2
    \in
    \{-k,\dots,k\}
  \end{equation}
  (so that this condition does not involve $Z_{n+1}$).
  Again by the local limit theorem,
  the maximum probability is asymptotically equivalent to
  \[
    \frac{2k+1}{\sqrt{2\pi n/4}}
    \sim
    \frac{2\sqrt{2}k}{\sqrt{\pi n}};
  \]
  it is attained at $\theta:=1/2$.
  Now the assumption $1\ll k\ll\sqrt{n}$ is essential
  in order for the exponential term in the local limit theorem to go away.
  Therefore, for any $G\in\ER$, there are $y_1,\dots,y_{n+1}\in\{0',1'\}$ such that
  \begin{equation*} % \label{eq:GGG}
    G(y_1,\dots,y_{n+1})
    \le
    \frac{b\sqrt{\pi n}}{2\sqrt{2}k}
  \end{equation*}
  (remember that we are omitting $x_0$),
  where $b>1$ is a constant
  (to be chosen close to $1$ later),
  and
  \begin{equation*} % \label{eq:condition}
    \sum_{i=1}^n y_i - n/2
    \in
    \{-k,\dots,k\}
  \end{equation*}
  (cf.~\eqref{eq:condition}).
  Fix such $y_1,\dots,y_{n+1}$ and set
  \begin{equation*} % \label{eq:y}
    y := \sum_{i=1}^n y_i - n/2
  \end{equation*}
  (cf.~\eqref{eq:y}).
  Taking $a$ and $b$ sufficiently close to 1,
  we obtain
  \begin{equation}\label{eq:to-explain}
    \frac{G(y_1,\dots,y_{n+1})}{E\i(y_1,\dots,y_n,y)}
    =
    (n+1)
    \frac{G(y_1,\dots,y_{n+1})}{E(y_1,\dots,y_n,y)}
    \le
    \frac{b\sqrt{\pi n}(n+1)}{2\sqrt{2}k a\e\sqrt{\pi/2}n^{3/2}}
    <
    \frac{c}{\e\left|\mathbf{Y}\right|}
  \end{equation}
  for a sufficiently large $k$.
  The equality ``$=$'' in \eqref{eq:to-explain} follows from
  $\{0',1'\}$ and $\{-k,\dots,k\}$ being disjoint sets,
  the inequality ``$\le$'' follows from the definitions of $E$ and $G$,
  and the inequality ``$<$'' follows from $\left|\mathbf{Y}\right|=2k+3$.
  This proves Theorem~\ref{thm:lower-1} since $E\i=E/E\X$.

  Finally, the formal meaning of the condition $1\ll k\ll\sqrt{n}$
  is that the first sentence in the statement of Theorem~\ref{thm:lower-1}
  can be replaced by
  ``For each constant $c>1$ there is $C>0$
  such that the following statement holds true assuming
  $\left|\mathbf{Y}\right|\ge C$ and $\sqrt{n}\ge C\left|\mathbf{Y}\right|$.''
% \end{proof}

\subsection{Proof of Theorem~\ref{thm:lower-2}}
\label{subapp:proof-lower-2}

% \begin{myproof}[of Theorem~\ref{thm:lower-2}]
  This section spells out details omitted in the proof sketch
  given in Sect.~\ref{sec:classification}.
  As mentioned there,
  the probability of \eqref{eq:Y-G} being at least $0.5$
  and the probability of \eqref{eq:Y-E'} being at least $0.502$
  follow immediately from Markov's inequality.
  The notation $Y_1,\dots,Y_{n+1}$ was introduced formally in Sect.~\ref{subapp:proof-lower-1}
  (and used informally already in the proof sketch of Theorem~\ref{thm:lower-2}).

  Now let us define $E$:
  \begin{wideformula}\label{eq:EE}
    E(y_1,\dots,y_{n+1})
    :=
    {}\NarrowLineBreak
    \begin{cases}
      c\left|\mathbf{Y}\right|
        & \text{if $S\equiv0\Mod{\left|\mathbf{Y}\right|}$
          and $\forall y\in\mathbf{Y}:
          \bigl|k_y-n/\left|\mathbf{Y}\right|\bigr| \le 0.1 n/\left|\mathbf{Y}\right|$}\\
      0 & \text{otherwise},
    \end{cases}
  \end{wideformula}
  where $S$ is the sum of $y_1,\dots,y_{n+1}$
  and $k_y:=\left|\left\{i\in\{1,\dots,n+1\}\mid y_i=y\right\}\right|$
  is the number of times $y\in\mathbf{Y}$ occurs in the sequence $y_1,\dots,y_{n+1}$.
  We are required to prove two statements for a large enough $n$:
  first, that $E\in\ER$,
  and second, that the probability of $E(Y_1,\dots,Y_n,Y)>0$ is at least 0.999
  (see \eqref{eq:Y-E}).

  The second statement is easier.
  Let $y_1,\dots,y_{n+1}$ be the values taken by $Y_1,\dots,Y_n,Y$, respectively.
  Then the first condition
  \begin{equation}\label{eq:first}
    S\equiv0\Mod{\left|\mathbf{Y}\right|}
  \end{equation}
  in \eqref{eq:EE} holds automatically,
  and the second condition
  \begin{equation}\label{eq:second}
    \forall y\in\mathbf{Y}:
    \bigl|k_y-n/\left|\mathbf{Y}\right|\bigr| \le 0.1 n/\left|\mathbf{Y}\right|
  \end{equation}
  holds with probability that tends to 1 as $n\to\infty$ for each $y\in\mathbf{Y}$,
  by the central limit theorem.
  Since $\mathbf{Y}$ is a finite set, the second statement has been proved.

  Now let us prove the first statement.
  Let $Q$ be a probability measure on $\mathbf{Y}$
  and set $R:=Q^{n+1}$.
  We will denote by $A$ the event given after ``if'' in \eqref{eq:EE},
  namely the conjunction of \eqref{eq:first} and \eqref{eq:second}.
  It suffices to prove that asymptotically as $n\to\infty$ the $R$-probability of $A$
  does not exceed $c^{-1/2}/\left|\mathbf{Y}\right|$, uniformly in $Q$.
  Consider two cases:
  \begin{itemize}
  \item
    If $Q(\{y\})\ge0.1/\left|\mathbf{Y}\right|$ for all $y\in\mathbf{Y}$,
    the convergence $R(A_1)\to1/\left|\mathbf{Y}\right|$ as $n\to\infty$
    for the superset $A_1$ of $A$ given by \eqref{eq:first}
    follows from the asymptotic uniformity result
    of \cite[Theorem~(5.01)]{Horton/Smith:1949}
    (proved independently in \cite[Theorem~2]{Dvoretzky/Wolfowitz:1951}).
    The convergence is uniform on this compact set of $Q$.
  \item
    If $Q(\{y\})\le0.1/\left|\mathbf{Y}\right|$ for some $y\in\mathbf{Y}$,
    the convergence $R(A_2)\to0$ as $n\to\infty$
    for the superset $A_2$ of $A$ given by \eqref{eq:second}
    follows from, e.g., the central limit theorem
    combined with the Bonferroni correction for the multiplicity of $y$.
    The convergence is again uniform on this compact set of $Q$.
  \end{itemize}
  Combining the two cases, we get $E\in\ER$.
% \end{myproof}

\section{Reducing randomness predictors to conformal}
\label{app:general}

In the main paper we assumed the train-invariance of randomness predictors.
In this appendix we drop this assumption,
since statistical principles may go astray
(see, e.g., the criticism \cite{Vovk:SS-local} of the conditionality principle
from the vantage point of conformal prediction).
Besides, in some cases we may have convincing reasons to violate train-invariance,
such as using inductive conformal predictors \cite[Sect.~4.2]{Vovk/etal:2022book}
for the purpose of computational efficiency.

\begin{figure}[bt]
  \begin{center}
    \unitlength 0.50mm
    \begin{picture}(80,80)(-10,-10)  % the size of the box and the coordinates of the bottom left corner
      \thinlines
      \put(0,0){\line(1,0){70}}  % the bottom line (thin)
      \put(70,0){\line(0,1){70}} % the right line (thin)
      \put(0,70){\line(1,0){70}} % the top line (thin)
      \put(0,0){\line(0,1){70}}  % the left line (thin)
      \put(15,15){\line(1,0){40}} % the bottom interior line (thin)
      \put(55,15){\line(0,1){40}} % the right interior line (thin and thick)
      \put(15,55){\line(1,0){40}} % the top interior line (thin and thick)
      \put(15,15){\line(0,1){40}} % the left interior line (thin)
      \put(0,0){\line(1,1){15}}     % the SW slanted line (thin)
      \put(70,0){\line(-1,1){15}}   % the SE slanted line (thin and thick)
      \put(70,70){\line(-1,-1){15}} % the NE slanted line (thin)
      \put(0,70){\line(1,-1){15}}   % the NW slanted line (thin and thick)
      \thicklines % now thick lines
      \put(55,15){\red{\line(0,1){40}}} % the right interior line (thin and thick)
      \put(15,55){\red{\line(1,0){40}}} % the top interior line (thin and thick)
      \put(70,0){\red{\line(-1,1){15}}} % the SE slanted line (thin and thick)
      \put(0,70){\red{\line(1,-1){15}}} % the NW slanted line (thin and thick)
      % and now the notation for the classes; first external:
      \put(-4,74){\makebox(0,0)[cc]{$\PR$}}  % top left corner (most important)
      \put(-5,-5){\makebox(0,0)[cc]{$\PtR$}} % bottom left corner
      \put(75,-5){\makebox(0,0)[cc]{$\PtX$}} % bottom right corner
      \put(75,75){\makebox(0,0)[cc]{$\PX$}}  % top right corner
      % and now internal:
      \put(16,54){\makebox(0,0)[tl]{$\ER$}}  % top left corner
      \put(16,16){\makebox(0,0)[bl]{$\EtR$}} % bottom left corner
      \put(54,16){\makebox(0,0)[br]{$\EtX$}} % bottom right corner
      \put(54,54){\makebox(0,0)[tr]{$\EX$}}  % top right corner
    \end{picture}
    \qquad\qquad\qquad
    \unitlength 0.50mm
    \begin{picture}(80,80)(-10,-10)  % the size of the box and the coordinates of the bottom left corner
      \thinlines
      \put(64,0){\vector(-1,0){58}}  % the bottom embedding
      \put(70,5){\vector(0,1){60}}  % the right embedding
      \put(64,70){\vector(-1,0){58}} % the top embedding
      \put(0,5){\vector(0,1){60}}   % the left embedding
      \put(49,15){\vector(-1,0){28}} % the bottom interior embedding
      \put(55,20){\vector(0,1){30}}  % the right interior embedding
      \put(49,55){\vector(-1,0){28}} % the top interior embedding
      \put(15,20){\vector(0,1){30}}  % the left interior embedding
      \put(5,5){\line(1,1){5}}     % the SW slanted line (calibration)
      \put(65,5){\line(-1,1){5}}   % the SE slanted line (calibration)
      \put(65,65){\line(-1,-1){5}} % the NE slanted line (calibration)
      \put(5,65){\line(1,-1){5}}   % the NW slanted line (calibration)
      % and now the notation for the classes; first external:
      \put(-6,65){\framebox(12,10)[cc]{$\PR$}}  % top left corner
      \put(-6,-5){\framebox(12,10)[cc]{$\PtR$}}  % bottom left corner
      \put(64,-5){\framebox(12,10)[cc]{$\PtX$}} % bottom right corner
      \put(64,65){\framebox(12,10)[cc]{$\PX$}} % top right corner
      % and now internal:
      \put(9,50){\framebox(12,10)[cc]{$\ER$}}  % top left corner
      \put(9,10){\framebox(12,10)[cc]{$\EtR$}} % bottom left corner
      \put(49,10){\framebox(12,10)[cc]{$\EtX$}} % bottom right corner
      \put(49,50){\framebox(12,10)[cc]{$\EX$}}  % top right corner
    \end{picture}
  \end{center}
  \caption{A cube representing the main classes of predictors considered
    in Appendix~\ref{app:general}.
    % The polygonal chain $\PR\to\ER\to\EX\to\EtX\to\PtX$
    % is shown in red in the left panel.
    The right panel shows embeddings as arrows,
    as in Figure~\ref{fig:simple}.
    \label{fig:classes}}
\end{figure}

The connections between various classes of predictors
(introduced in the main paper)
that we use or explore in this appendix are shown
in red in the left-hand panel of Figure~\ref{fig:classes}.
Namely, we are interested in connections between:
\begin{enumerate}[label=(\alph*)] % [(a)]
\item\label{it:c1}
  randomness e-predictors and exchangeability e-predictors;
\item\label{it:c2}
  exchangeability e-predictors and conformal e-predictors;
\item\label{it:c3}
  randomness p-predictors and randomness e-predictors;
\item\label{it:c4}
  conformal e-predictors and conformal p-predictors.
\end{enumerate}
Theorem~\ref{thm:Kolmogorov} stated in Subsect.~\ref{subsec:e}
is general enough to cover connection~\ref{it:c1}.
Connection~\ref{it:c2} will be discussed in Sect.~\ref{subapp:t}.
Connections~\ref{it:c3} and~\ref{it:c4}
result from the possibility of converting p-values to e-values and back,
as discussed in Subsect.~\ref{subsec:p}.

In Sect.~\ref{subapp:putting} we will see how the connections
shown in red in Figure~\ref{fig:classes}
can be combined to demonstrate the universality of conformal prediction
without accepting the train-invariance principle.
Section~\ref{subapp:class} applies this to classification.
But we start in Sect.~\ref{subapp:operator-t}
by introducing an operator making randomness predictors
(in particular, exchangeability predictors)
train-invariant.

\subsection{Another operator}
\label{subapp:operator-t}

In addition to the operators \eqref{eq:i} and \eqref{eq:X},
we will need another operator,
\begin{equation}\label{eq:t}
  E\t(z_1,\dots,z_{n+1})
  :=
  \frac{1}{n!}
  \sum_{\sigma}
  E(z_{\sigma(1)},\dots,z_{\sigma(n)},z_{n+1}),
\end{equation}
$\sigma$ ranging over the permutations of $\{1,\dots,n\}$.
Now we have $E\t\in\EtX$ (resp.\ $E\t\in\EtR$)
whenever $E\in\EX$ (resp.\ $E\t\in\ER$).
Both operators~\eqref{eq:i} and~\eqref{eq:t} are kinds of averaging:
while $E\mapsto E\i$ averages over all permutations of an input data sequence
(including both training and test examples),
$E\mapsto E\t$ averages over the permutations of the training sequence only.
And while $E\mapsto E\i$ is polynomially computable
for train-invariant $E$,
$E\mapsto E\t$ requires exponential computation time in general.

Using two of the three operators \eqref{eq:i}, \eqref{eq:X}, and \eqref{eq:t},
we can turn any randomness e-variable $E$
to an exchangeability e-variable $E\X$
to a conformal e-variable $(E\X)\t$.
The following lemma shows that the order
in which the last two operators are applied does not matter.

\begin{lemma}\label{lem:commute}
  The operators ${}\t$ and ${}\X$ commute:
  for any $E\in\EEE$, $(E\t)\X=(E\X)\t$.
\end{lemma}

\begin{proof}
  Let us fix a data sequence $z_1,\dots,z_{n+1}$ and check
  \[
    (E\t)\X(z_1,\dots,z_{n+1})=(E\X)\t(z_1,\dots,z_{n+1}).
  \]
  As functions of a permutation of $z_1,\dots,z_{n+1}$,
  $E$ and $E\X$ are proportional to each other,
  and therefore, $E\t$ and $(E\X)\t$ are also proportional to each other.
  This implies $(E\t)\X=(E\X)\t$ on the permutations of $z_1,\dots,z_{n+1}$.
  And this is true for each $(z_1,\dots,z_{n+1})$.
\end{proof}

\noindent
We will let $\tX$ stand for the composition of the two operators:
\[
  E\tX
  :=
  (E\t)\X
  =
  (E\X)\t.
\]
Let us say that $E\in\EX$ is \emph{admissible}
if \eqref{eq:EX} always holds with ``$=$'' in place of ``$\le$''.
(This agrees with the standard notion of admissibility
in statistical decision theory.)
The intuition (which can be formalized easily)
behind the four operators that we have introduced in this paper is that:
\begin{itemize}
\item
  ${}\i$ projects $\ER$ onto $\EiR$;
  it also projects the admissible part of $\EX$ onto the identical 1;
\item
  ${}\X$ projects $\EEE$ onto the admissible part of $\EX$;
\item
  ${}\t$ projects $\EX$ onto $\EtX$ and $\ER$ onto $\EtR$;
\item
  ${}\tX$ projects $\EEE$ onto the admissible part of $\EtX$.
\end{itemize}
In particular, these operators are idempotent:
\[
  (E\i)\i = E\i,
  \quad
  (E\X)\X = E\X,
  \quad
  (E\t)\t = E\t,
  \quad
  (E\tX)\tX = E\tX
\]
(and we can allow any $E\in\EEE$ here).
Despite these operators being projections,
we cannot claim that these ways of moving between different classes of predictors
are always optimal.

Lemma~\ref{lem:commute} lists the only two cases
where the combination of two of our three basic operators
(${}\i$, ${}\X$, and ${}\t$)
gives something interesting.
The other four cases are:
\[
  (E\X)\i = (E\i)\X = 1,
  \quad
  (E\i)\t = (E\t)\i = E\i.
\]

\subsection{Efficiency of making predictors train-invariant}
\label{subapp:t}

To state our result in its strongest form,
we define a \emph{test-conditional exchangeability e-variable}
$G=G(z_1,\dots,z_n,z_{n+1})$
as an element of $\EEE$ satisfying
\[
  \forall(z_1,\dots,z_{n+1}) \; \forall\sigma:
  \frac{1}{n!}
  \sum_{\sigma}
  G(z_{\sigma(1)},\dots,z_{\sigma(n)},z_{n+1})
  \le
  1,
\]
$\sigma$ ranging over the permutations of $\{1,\dots,n\}$.
Such $G$ form a subclass of $\EX$ (and therefore, of $\ER$).

\begin{theorem}\label{thm:train-invariance}
  Let $B:\mathbf{Z}\hookrightarrow\mathbf{Y}$ be a Markov kernel.
  For each exchangeability e-predictor $E$,
  \begin{equation}\label{eq:thm-train-invariance}
    G(z_1,\dots,z_n,z_{n+1})
    :=
    \int
    \frac
      {E(z_1,\dots,z_n,x_{n+1},y)}
      {E\t(z_1,\dots,z_n,x_{n+1},y)}\;
    B(\d y\mid z_{n+1})
  \end{equation}
  (with $0/0$ interpreted as 0)
  is a test-conditional exchangeability e-variable.
\end{theorem}

The interpretation of \eqref{eq:thm-train-invariance}
is similar to that of \eqref{eq:Kolmogorov}.
It will be clear from the proof that we can allow $E$ to be any randomness e-predictor,
or even any element of $\EEE$.

\begin{proof}[Proof of Theorem~\ref{thm:train-invariance}]
  Let us check that the right-hand side of \eqref{eq:thm-train-invariance}
  is a test-conditional exchangeability e-variable:
  \begin{multline*}
    \frac{1}{n!}
    \sum_{\sigma}
    G(z_{\sigma(1)},\dots,z_{\sigma(n)},z_{n+1})
    \iftoggle{CONF}{}{\\}
    =
    \frac{1}{n!}
    \sum_{\sigma}
    \int
    \frac
      {E(z_{\sigma(1)},\dots,z_{\sigma(n)},x_{n+1},y)}
      {E\t(z_1,\dots,z_n,x_{n+1},y)}
    B(\d y\mid z_{n+1})\\
    =
    \int
    \frac
      {E\t(z_1,\dots,z_n,x_{n+1},y)}
      {E\t(z_1,\dots,z_n,x_{n+1},y)}
    B(\d y\mid z_{n+1})
    \le
    1.
    \qedhere
  \end{multline*}
\end{proof}

\subsection{Putting everything together}
\label{subapp:putting}

The following theorem combines Theorems~\ref{thm:Kolmogorov} and~\ref{thm:train-invariance}
and establishes a connection between randomness and conformal e-predictors,
without accepting the train-invariance principle.
Remember that the conformal e-predictor $\EtX$ derived from a randomness e-predictor $E$
is obtained by combining the operators \eqref{eq:X} and \eqref{eq:t},
i.e., as
\begin{equation}
  \notag
  E\tX(z_1,\dots,z_{n+1})
  :=
  (n+1)
  \frac
    {\sum_{\sigma}E(z_{\sigma(1)},\dots,z_{\sigma(n)},z_{n+1})}
    {\sum_{\pi}E(z_{\pi(1)},\dots,z_{\pi(n+1)})},
\end{equation}
$\sigma$ and $\pi$ ranging over the permutations
of $\{1,\dots,n\}$ and $\{1,\dots,n+1\}$, respectively.

\begin{corollary}\label{cor:e-main}
  Let $B:\mathbf{Z}\hookrightarrow\mathbf{Y}$ be a Markov kernel.
  For each randomness e-predictor $E$,
  \begin{equation}\label{eq:cor-e-main}
    G(z_1,\dots,z_n,z_{n+1})
    :=
    \e^{-1/2}
    \int
    \sqrt{\frac
      {E(z_1,\dots,z_n,x_{n+1},y)}
      {E\tX(z_1,\dots,z_n,x_{n+1},y)}}\;
    B(\d y\mid z_{n+1})
  \end{equation}
  is a randomness e-variable.
\end{corollary}

The main weakness of Corollary~\ref{cor:e-main} is the presence of the term $\e^{-1/2}$,
but it might be inevitable.

\begin{proof}
  Applying the Cauchy--Schwarz inequality, we have, for some $G_1,G_2,G_3\in\ER$,
  \begin{multline*}
    \e^{-1/2}
    \int
    \sqrt{\frac
      {E(z_1,\dots,z_n,x_{n+1},y)}
      {E\tX(z_1,\dots,z_n,x_{n+1},y)}}\;
    B(\d y\mid z_{n+1})\\
    =
    \e^{-1/2}
    \int
    \sqrt{\frac
      {E(z_1,\dots,z_n,x_{n+1},y)}
      {E\X(z_1,\dots,z_n,x_{n+1},y)}}
    \enspace
    \sqrt{\frac
      {E\X(z_1,\dots,z_n,x_{n+1},y)}
      {E\tX(z_1,\dots,z_n,x_{n+1},y)}}\;
    B(\d y\mid z_{n+1})\\
    \le
    \sqrt{
      \e^{-1}
      \int
      \frac
        {E(z_1,\dots,z_n,x_{n+1},y)}
        {E\X(z_1,\dots,z_n,x_{n+1},y)}\;
      B(\d y\mid z_{n+1})}
    \iftoggle{CONF}{\enspace}{\\\quad{}\times}
    \sqrt{
      \int
      \frac
        {E\X(z_1,\dots,z_n,x_{n+1},y)}
        {E\tX(z_1,\dots,z_n,x_{n+1},y)}
      B(\d y\mid z_{n+1})
    }\\
    =
    \sqrt{G_1(z_1,\dots,z_{n+1})G_2(z_1,\dots,z_{n+1})}
    \le
    G_3(z_1,\dots,z_{n+1})
  \end{multline*}
  (the existence of $G_1$ and $G_2$ follows
  from Theorems~\ref{thm:Kolmogorov} and~\ref{thm:train-invariance},
  respectively,
  and $G_3$ can be set, e.g., to the arithmetic average of $G_1$ and $G_2$).
\end{proof}

Similarly to Corollary~\ref{cor:p-t},
we can adapt Corollary~\ref{cor:e-main} to p-predictors.

\begin{corollary}\label{cor:p-main}
  Let $B:\mathbf{Z}\hookrightarrow\mathbf{Y}$ be a Markov kernel
  and let $\delta\in(0,1)$.
  For each randomness p-predictor $P$
  there exists a conformal p-predictor $P'$ such that
  \begin{equation}\label{eq:cor-p-main}
    G(z_1,\dots,z_n,z_{n+1})
    :=
    \sqrt{\frac{\delta}{\e}}
    \int
    \sqrt{\frac
      {P'(z_1,\dots,z_n,x_{n+1},y)}
      {P^{1-\delta}(z_1,\dots,z_n,x_{n+1},y)}}
    \:
    B(\d y\mid z_{n+1})
  \end{equation}
  is a randomness e-variable.
\end{corollary}

The interpretation of \eqref{eq:cor-p-main} is similar to that of \eqref{eq:cor-p-t}:
$P'(z_1,\dots,z_n,x_{n+1},y)$ is typically small when $P(z_1,\dots,z_n,x_{n+1},y)$ is small.
Instead of \eqref{eq:bounded-t} we will have
\begin{equation*} % \label{eq:bounded}
  \sqrt{\frac{\delta}{\e}}
  \sqrt{\frac
    {P'(z_1,\dots,z_n,x_{n+1},y)}
    {P^{1-\delta}(z_1,\dots,z_n,x_{n+1},y)}}
  <
  \frac{1}{\epsilon_1\epsilon_2},
\end{equation*}
which can be rewritten as
\[
  P'(z_1,\dots,z_n,x_{n+1},y)
  <
  \frac{\e}{\delta\epsilon_1^2\epsilon_2^2}
  P^{1-\delta}(z_1,\dots,z_n,x_{n+1},y)
\]
in place of~\eqref{eq:rewrite}.

\begin{proof}[Proof of Corollary~\ref{cor:p-main}]
  We proceed as in the proof of Corollary~\ref{cor:p-t}
  except for replacing \eqref{eq:Kolmogorov} by \eqref{eq:cor-e-main}.
\end{proof}

\subsection{Applications to classification}
\label{subapp:class}

As compared with Sect.~\ref{sec:classification},
we get similar but weaker performance guarantees for the derived conformal predictors
without accepting the train-invariance principle.
In the binary case $\mathbf{Y}=\{-1,1\}$, we apply \eqref{eq:cor-e-main}
to obtain
\begin{equation*}
  G(z_1,\dots,z_n,z_{n+1})
  :=
  \e^{-1/2}
  \sqrt{\frac
    {E(z_1,\dots,z_n,x_{n+1},-y_{n+1})}
    {E\tX(z_1,\dots,z_n,x_{n+1},-y_{n+1})}}
\end{equation*}
in place of \eqref{eq:binary}.
This implies \eqref{eq:epsilon} with $\EtX$ in place of $\EX$
and $\epsilon^2$ in place of $\epsilon$.

Instead of \eqref{eq:multi-class} now we have
\begin{equation*}
  G(z_1,\dots,z_n,z_{n+1})
  :=
  \frac{\e^{-1/2}}{\left|\mathbf{Y}\right|-1}
  \sum_{y\in\mathbf{Y}\setminus\{y_{n+1}\}}
  \sqrt{\frac
    {E(z_1,\dots,z_n,x_{n+1},y)}
    {E\tX(z_1,\dots,z_n,x_{n+1},y)}},
\end{equation*}
and instead of~\eqref{eq:epsilon-worst} we have
\begin{equation*}
  \forall y\in\mathbf{Y}\setminus\{y_{n+1}\}:
  E\tX(z_1,\dots,z_n,x_{n+1},y)
  >
  \frac{\e^{-1}\epsilon^2}{(\left|\mathbf{Y}\right|-1)^2}
  E(z_1,\dots,z_n,x_{n+1},y).
\end{equation*}
\end{document}